\documentclass[11pt]{article}
\usepackage[english]{babel}
\usepackage[utf8x]{inputenc}
\usepackage[T1]{fontenc}

\usepackage[letterpaper,top=1in,bottom=1in,left=1in,right=1in,marginparwidth=1.75cm]{geometry}

\usepackage{amsmath,amsfonts,amsthm,amssymb,mathrsfs,dsfont} % Math packages
\usepackage{mathtools}
\usepackage{graphicx}
\usepackage[colorinlistoftodos]{todonotes}
\usepackage[colorlinks=true, allcolors=blue]{hyperref}
\usepackage[capitalize,nameinlink]{cleveref}
\usepackage{verbatim}

\newcommand{\R}{\mathbb{R}}
\newcommand{\eps}{\varepsilon}

\newcommand{\vnote}[1]{\textcolor{red}{\small {\textbf{(Vishesh: }#1\textbf{) }}}}
\newcommand{\fnote}[1]{\textcolor{blue}{\small {\textbf{(Fred: }#1\textbf{) }}}}

\newtheorem{theorem}{Theorem}[section]
\newtheorem*{namedtheorem}{\theoremname}
\newcommand{\theoremname}{testing}

\newtheorem{thm}[theorem]{Theorem}
\newtheorem{lemma}[theorem]{Lemma}

\newtheorem*{question*}{Question}

\theoremstyle{definition}

\newtheorem{defn}[theorem]{Definition}
\newtheorem{remark}[theorem]{Remark}

\theoremstyle{plain}

\usepackage{algpseudocode,algorithm,algorithmicx}

\title{Approximating Partition Functions in Constant Time}
\author{Vishesh Jain\thanks{Massachusetts Institute of Technology. Department of Mathematics. Email: {\tt visheshj@mit.edu}} \and Frederic Koehler\thanks{Massachusetts Institute of Technology. Department of Mathematics. Email: {\tt fkoehler@mit.edu}} \and Elchanan Mossel\thanks{Massachusetts Institute of Technology. Department of Mathematics and IDSS. Supported by ONR grant N00014-16-1-2227   and 
NSF CCF-1665252 and DMS-1737944. Email: {\tt elmos@mit.edu} } }
\date{}
\begin{document}
\maketitle
% Skip title page in numbering
\thispagestyle{empty}
\setcounter{page}{0}

\begin{abstract}
We study approximations of the partition function of dense graphical models. 
Partition functions of graphical models play a fundamental role is statistical physics, in statistics and in machine learning. Two of the main methods for approximating the partition function are Markov Chain Monte Carlo and Variational Methods. An impressive body of work in mathematics, physics and theoretical computer science provides conditions under which Markov Chain Monte Carlo methods converge in polynomial time. These methods often lead to polynomial time approximation algorithms for the partition function in cases where the underlying model exhibits correlation decay. There are very few theoretical guarantees for the performance of variational methods. One exception is recent results by Risteski (2016) who considered {\em dense graphical models} and 
showed that using variational methods, it is possible to find an $O(\epsilon n)$  additive approximation to the {\em log partition function}
in time $n^{O(1/\epsilon^2)}$ even in a regime where correlation decay does not hold.

We show that under essentially the same conditions, an $O(\epsilon n)$  additive approximation of the log partition function can be found in {\em constant time}, independent of $n$. In particular, our results cover dense Ising and Potts models as well as dense graphical models with $k$-wise interaction. They also apply for low threshold rank models. 

To the best of our knowledge, our results are the first to give a constant time approximation to log partition functions and the first to use the algorithmic regularity lemma for estimating partition functions. As an application of our results we derive a constant time algorithm for approximating the magnetization of Ising and Potts model on dense graphs.

\end{abstract}

\newpage

\section{Introduction}
One of the major algorithmic tasks in the areas of Markov Chain Monte Carlo, in statistical inference and in machine learning is approximating partition functions of graphical models. The partition function or the approximate partition function are used in computing marginals and posteriors - two of the basic inference tasks in graphical models. Moreover, there is an intimate connection between computing the partition function and Markov Chain Monte Carlo methods. In particular, Jerrum and Sinclair~\cite{JerrumSinclair:89b} showed that it is possible to approximate the partition function for 
``self-reducible'' models for which a rapidly mixing Markov chain exists. On the other hand, for such models, a 
$(1+\eps)$ approximation of the partition function results in a rapidly mixing chain. 
Some key results in the theory of MCMC provide conditions for the existence of a rapidly 
mixing chains and therefore allow for efficient approximations of the partition functions e.g.~\cite{JerrumSinclair:89,JerrumSinclair:90,JeSiVi:04} and follow up work.  

A different line of work is based on the fact that 
the measure associated to any graphical model can be characterized by a variational problem. 
Thus, approximating the partition function can be achieved by solving optimization problems (e.g. \cite{blei-etal-variational}, \cite{wainwright-jordan-variational},  \cite{yedida-etal-bp}). Until recently, there were very few general guarantees for the performance of variational methods. One exception is results by Risteski~\cite{risteski-ising} who showed that using variational methods, it is possible to find an $O(\eps ||\vec{J}||_1)$  additive approximation to the {\em log partition function}  of
{\em dense graphical models} in time $n^{1/\eps^2}$, even in an interesting regime where correlation decay does not hold. Here, $||\vec{J}||_1$ denotes the total weight of all interactions of the graph so that we are guaranteed %a-priori that that log partition function which imply the 
an \emph{a priori} upper bound on the log partition function being $O(n + ||\vec{J}||_1)$. 

Following Risteski~\cite{risteski-ising}, we consider dense graphical models and show that it is possible to find an $O(\eps ||\vec{J}||_1)$  additive approximation to the {\em log partition function} in (randomized) {\em constant time}. Thus our main result stated informally is: 
\begin{thm} \label{thm:main_informal}
For dense graphical models, it possible to compute the log partition function with additive error 
$O(\eps ||\vec{J}||_1)$ in time that depends only on $\eps$ and the density parameter. 
\end{thm}  

We note that the approximation guarantee in Theorem~\ref{thm:main_informal} is rather weak. For most inference tasks, one is interested in good approximation of the partition function itself and our results (like those of~\cite{risteski-ising}) only provide an approximation of the partition function within a $\exp(\pm \epsilon n)$ multiplicative factor. Such an approximation is not useful for most inference and marginalization tasks. We note however that 
\begin{itemize}
\item It is easy to see that in constant time, it is impossible to get a sub-exponential approximation of the partition  function (see~\cref{thm-qualitative-tightness} for a formal statement). 

\item There are features of graphical models that are reflected by the log partition function. 
In particular, we show in~\cref{thm-approx-magnetization} how to utilize our results 
to obtain an $\eps$-approximation for the magnetization per-site for Ising models in constant time. Similar statements hold for other dense graphical models as well. The exact statement of ~\cref{thm-approx-magnetization} provides the magnetization for a model that has parameters close to the model we are interested in. We show that this is necessary as in constant time, it is impossible to break symmetry between phases. 
\end{itemize}

We note that our results hold for very general Markov random fields. On the other hand, even in the case of ferromagnetic Ising models, only polynomial time guarantees have been previously provided (see \cite{JerrumSinclair:90} and 
recent work~\cite{Barvinok:17,lss-deterministic}).  
%\fnote{TODO: somebody needs to finish this?}

%{\bf Fred: can you please add the ref and write this paragraph} 
We also note that our results can be thought as generalizing classical results
about the approximability of Max-Cut and similar problems on dense graphs -- in fact, the Max-Cut problem is just the special
case of computing the log partition function of an antiferromagnetic Ising model (i.e.
all non-zero entries of $J$ are negative) on a graph with equal edge weights
in the limit as $\|\vec J\|_1 \to \infty$ (so that the entropy
term becomes negligible). Note that in this case, the log-partition function is $\Omega(\|\vec J\|_1)$ (consider
choosing the $X_i$ according to independent Rademacher random variables)
so that an $\epsilon \|\vec J\|_1$-additive approximation actually gives a PTAS. The first PTAS for Max-Cut on dense graphs was given by Arora,
Karger and Karpinski \cite{akk} but ran in time $O(n^{\tilde O(1/\epsilon^2)})$; this was improved
to $O(2^{\tilde O(1/\epsilon^2)}$ by Frieze and Kannan in \cite{frieze-kannan-old}
using their weak regularity lemma, and remains essentially the fastest runtime
known (our algorithm matches this runtime). More recent works \cite{alon-etal-samplingCSP,mathieu-schudy} have shown that the \emph{sample complexity} of the problem -- the number of entries of the adjacency matrix which need to be probed -- is only $O(1/\epsilon^4)$. In our case, determining the correct sample complexity remains an interesting open problem (see \cref{tightness-section}). It is also interesting to note that in the optimization case, algorithms based on regularity and algorithms based on  convex programming hierarchies (see \cite{maxcut-lp,yoshida-zhou}) have the same dependence on $\epsilon$ (but not $n$) -- our result and \cite{risteski-ising} also have the same dependence on $\epsilon$ and $\Delta$, showing that this phenomena extends to counting-type problems as well, and suggesting that it will be difficult to beat a runtime of $2^{\tilde O(1/\epsilon^2\Delta)}$.
%Mention how classic $(1 + \epsilon)$-approximation to max-cut on dense graphs is special case where $|J_T| \to \infty$, etc.

\section{Overview of results}

\subsection{Notation and definitions}
We will consider Ising models where the spins are $\pm 1$ valued. The interactions can be both ferromagnetic and anti-ferromagnetic and of varying strengths. For simplicity, we primarily restrict ourselves to the case where there are no external fields, though it will be clear that our methods extend to that case as well (\cref{rmk-general-mrf}).  
\begin{defn} An \emph{Ising model} is
a probability distribution on the discrete cube $\{\pm1\}^n$ of the form
\[ \Pr[X = x] = \frac{1}{Z} \exp(\sum_{i,j} J_{i,j} x_i x_j) = \frac{1}{Z} \exp(x^T J x) \]
where the collection $\{J_{i,j}\}_{i,j\in\{1,\dots,n\}}$ are the entries of
an arbitrary real, symmetric matrix. 
The normalizing constant $Z=\sum_{x\in\{\pm1\}^{n}}\exp(\sum_{i,j=1}^{n}J_{i,j}x_{i}x_{j})$
is called the \emph{partition function }of the Ising model.
\end{defn}
We will often need to consider an $m\times n$ matrix $M$ as an $mn$-dimensional
vector (exactly \emph{how }we arrange the $mn$ entries of $M$ in
a single vector will not matter). When we do so, we will denote the
resulting vector by $\vec{M}$ for clarity. 
\begin{defn}
For a vector $v=(v_{1},\dots,v_{n})\in\R^{n}$ and $p\in[1,\infty)$,
the $L^{p}$ norm of $v$, denoted by $||v||_{p}$, is defined by
$||v||_{p}:=[|v_{1}|^{p}+\dots+|v_{n}|^{p}]^{1/p}$. We also
define $||v||_{\infty}:=\max_{i\in\{1,\dots,n\}}|v_{i}|$. 
\end{defn}
By viewing a matrix $M$ as a vector $\vec{M}$, these vector norms
will give rise to \emph{matrix norms }which we will use in the sequel.
The cases $p=1,2,\infty$ will be of particular importance to us. 

Another matrix norm that we will need arises from viewing an $m\times n$
matrix as a linear operator from $\R^{n}$ to $\R^{m}$ via matrix
multiplication. 
\begin{defn}
For an $m\times n$ matrix $M$, we define $||M||_{\infty\mapsto1}:=\sup_{x\in\R^{n},||x||_{\infty}\leq1}||Mx||_{1}$. 
\end{defn}

Following \cite{risteski-ising}, we will focus on $\Delta$-dense Ising models. 
\begin{defn}
An Ising model is \emph{$\Delta$-dense} if $\Delta ||\vec{J}||_{\infty} \le \frac{||\vec{J}||_1}{n^2}$.
\end{defn}

From a combinatorial perspective
$\Delta$-dense models are a generalization of dense graph models. 
From a statistical physics perspective, they generalize complete graph models such as the 
Curie-Weiss model and the SK model~\cite{SK:75}.

Our main results will provide algorithms which run in \emph{constant
time. } In order to provide such guarantees on problems with unbounded input size, we will work under the usual assumptions on the computational model for sub-linear algorithms (as in e.g. \cite{alon-etal-samplingCSP,frieze-kannan-matrix,indyk1999sublinear}). Thus, we can probe matrix entry $A(i,j)$ in $O(1)$ time. Also note that by the standard Chernoff bounds, it follows that for any set of vertices $V$ for which we can test membership in $O(1)$ time, we can also estimate $|V|/n$ to additive error $\epsilon$ w.h.p. in constant time using $\tilde{O}(1/\epsilon^2)$ samples. This approximation will always suffice for us; so, for the sake of clarity in exposition, we will henceforth ignore this technical detail and just assume that we have access to $|V|/n$ (as in e.g. \cite{frieze-kannan-matrix}).
\iffalse
the \emph{Probe }model of computation,
in which we assume that given a pair of indices $(i,j)$ and a matrix
$A$, we can determine $A(i,j)$ in $O(1)$ time. The same model of
computation was also used in \cite{frieze-kannan-matrix}. We will also assume that we can perform
basic arithmetic operations and random sampling in $O(1)$ time. 
\fi

\subsection{Main results}
%%\vnote{Please check the statements of the theorems! I came up with them lazily}
%% We were too lazy to do this
%\fnote{Maybe we can just give short ``informal'' statements here and defer ``full'' statements
%to the text? E.g. something like: There is an algorithm running in time $2^{\tilde O(1/\epsilon^2\Delta)}$ which computes an $\epsilon \|\vec J\|_1$ additive
%approximation to $\log Z$ for $n$ sufficiently large, with high probability and for $n$ sufficiently large.}
%Before stating our main results, we recall that when $||\vec{J}||_{\infty} < 1/n$, Dobrushin's uniqueness criterion implies that there is correlation decay 
%and hence the Glauber dynamics mixes randomly
%\cite{dobrushin-uniqueness}, so Markov chain based methods provide much stronger guarantees than our techniques. Therefore, we will restrict ourselves to the case $||\vec{J}||_{\infty} \geq 1/n$ in the sequel.\footnote{However, this assumption is not essential; if we are willing to use $O(\log n)$ runtime, our algorithm can give an $\epsilon \|\vec J\|_1$ approximation as long as $\|\vec J\|_{\infty} = \omega(\frac{\log n}{n^2})$. This is done just by setting $\lambda = O(\epsilon/n)$ in the proof of Theorem~\ref{thm-Delta-dense}.}
%\begin{comment}
%\fnote{Our actual theorem statements (section 6) rely on assumptions about what
%the interesting parameter regime is for $\|\vec J\|_1$, which needs to be explained by citing Dobrushyn uniqueness. }
%\end{comment}

Our main theorem for $\Delta$-dense Ising models is:
\begin{theorem}\label{thm-Delta-dense} 
%\begin{comment}
%For $\Delta$-dense Ising models, there exists a randomized algorithm
%$A$ such that given $\epsilon,\delta>0$, $A$ computes with probability
%$\geq1-\delta$ an $\epsilon J_{T}$ additive approximation to $\log Z$
%in time $2^{O(\frac{1}{\Delta\epsilon^{2}})}$. 
%\end{comment}
There exists a universal constant $C$ such that for any $\epsilon, \delta > 0$, and for any $\Delta$-dense Ising model $J$ on $n$ nodes with $n \geq 2^{C/\epsilon^2 \Delta}$ 
%and $||\vec{J}||_{\infty} \ge 1/n$, 
there is an algorithm running in time $2^{\tilde O(1/\epsilon^2\Delta)}\log(1/\delta)$ which computes an $\epsilon (n+\|\vec J\|_1)$ additive approximation
to $\log Z$ with probability at least $1 - \delta$.
\end{theorem}
Our methods extend in a straightforward manner to higher order Markov random fields on general finite alphabets as well. We will review the definitions in \cref{mrf}. For simplicity, we only state the result for the case of binary $k$-uniform Markov random fields. 
%As above, we restrict our attention to the case when $||\vec{J}||_{\infty} \geq 1/n^{k-1}$. 
\begin{theorem}\label{thm-mrf} 
Fix $k \geq 3$. There exists a universal constant $C$ such that for any $\epsilon, \delta > 0$, and for any $\Delta$-dense binary $k$-uniform Markov random field $J$ on $n$ with $n\geq 2^{C(\epsilon\sqrt{\Delta})^{2-2k}} $ 
%and  $||\vec{J}||_{\infty} \ge 1/n^{k-1}$, 
there is an algorithm running in time $2^{\tilde O(\epsilon\sqrt{\Delta})^{2-2k}}\log(1/\delta)$ which computes an $\epsilon(\|\vec J\|_1 + n)$ additive approximation
to $\log Z$ with probability at least $1 - \delta$.
\end{theorem}
In the previous theorem, it is possible to improve the dependence
on $\epsilon$ at the expense of introducing a factor of $n^{k}$
in the running time. We have: 
\begin{theorem}\label{thm-mrf-nonconstant}
Fix $k \geq 3$. There exists a universal constant $C$ such that for any $\epsilon, \delta > 0$, and for any $\Delta$-dense binary $k$-uniform Markov random field $J$ on $n$ with $n\geq 2^{C/\epsilon^{2}\Delta}$ 
% and $||\vec{J}||_{\infty} \ge 1/n^{k-1}$, 
there is an algorithm running in time $2^{\tilde{O}(1/\epsilon^{2}\Delta)}\log(1/\delta)n^{k}$ which computes an $\epsilon (\|\vec J\|_1 + n)$ additive approximation
to $\log Z$ with probability at least $1 - \delta$.
\end{theorem}

\begin{remark}\label{rmk-general-mrf}
Although for simplicity, we stated \cref{thm-mrf} and \cref{thm-mrf-nonconstant} only for $k$-uniform $\Delta$-dense Markov random fields, it is immediately seen that the results extend to general $\Delta$-dense Markov random fields of order $K$ by simply applying \cref{reg-fk-higher} or \cref{reg-alon-etal} to each $J^{(k)}$, $k\in[K]$. Again, see \cref{mrf} for definitions. In particular, this directly handles dense Ising models with external fields.
\end{remark}

\begin{remark}\label{rmk-J-bd}
The error term in~\cref{thm-Delta-dense} has two terms $n$ and $\|\vec J\|_1$. 
%By the density condition, 
If $\|\vec J\|_1 \leq \eta n$, 
%then $\| \vec{J} \|_{\infty} = o(\eta/n)$, in which case 
then $|\log \exp(x^T J x)| = O(\eta n)$ and therefore, $\log Z = n( \log 2 + O(\eta))$.  
%the Ising model satisfies the Dobrushin condition~\cite{dobrushin-uniqueness} which implies decay of correlation and also that $\log Z = (1+o(1)) n \log_2$. 
Similar statements hold for the other theorems as well. Thus, in interesting applications of the theorem, we may assume WLOG that the dominant term is $\|\vec J\|_1$ (by dominant here, we mean it is at least a small fraction of the other term). 
\end{remark} 

As in \cite{risteski-ising}, we are also able to handle the case of Ising models whose interaction matrices have low threshold ranks, although not in constant time. Again, for simplicity, we only record the result for regular low threshold rank models. Definitions will be presented in \cref{ltr}. Following \cref{rmk-J-bd}, we assume $J'\geq 1$. 

\begin{theorem}\label{thm-ltr}
There exists a universal constant $C$ such that for any $\epsilon > 0$, and for any regular Ising model $J$ on $n$ nodes with $n \geq 2^{Ct/\epsilon^2}$ (here, t is the $\epsilon/2$ sum-of-squares rank of the model) and $J' \geq 1$, there is an algorithm running in time $2^{\tilde O(t/\epsilon^2)} + poly(n,t,\frac{1}{\epsilon})$ which computes an $\epsilon \|\vec J\|_1$ additive approximation to $\log Z$. 
\end{theorem}

\subsection{Approximating the expected total magnetization}
%\enote{I think the proofs of this section should go to the end of the paper but we should include a statement that is something like the following:}

Given an Ising model $\Pr[X=x]=\frac{1}{Z}\exp\{\sum_{i,j}J_{i,j}x_{i}x_{j}+\sum_{i}h_{i}x_{i}\}$,
one of the most fundamental questions one can ask about it is how
many spins are $+1$ and $-1$ in expectation. In the case of the ferromagnetic
Ising model ($J_{i,j}\geq0$ for all $i,j$), this corresponds to
how strongly the system is magnetized. Accordingly, we
define the \emph{(expected total) magnetization }of an
Ising model by $\boldsymbol{E}[\sum_{i}x_{i}]$,
%, and the \emph{expected magnetization per site} by $\frac{1}{n}\boldsymbol{E}[\sum_{i}x_{i}]$, 
where the expectation
is with respect to the Ising distribution on $\{\pm1\}^{n}$. Perhaps
surprisingly, our results easily show that for dense
Ising models, one can obtain, in some sense, an approximation to the expected
magnetization in constant time. More precisely, we have:
\begin{theorem}\label{thm-approx-magnetization}
Consider a %ferromagnetic 
$\Delta$-dense Ising Model 
%on a dense graph 
\[
\Pr[X=x]:=\frac{1}{Z}\exp\{\sum_{i,j}J_{i,j}x_{i}x_{j} + \sum_i h_i x_i\}, \quad
\Delta ||\vec{J}||_{\infty} \le \frac{||\vec{J}||_1}{n^2}, \quad 
\Delta ||h||_{\infty} \le \frac{||\vec{h}||_1}{n}
\]
%where the $J_{i,j}$ are all positive and the model is dense. 
Let 
\[
\Pr_{h}[X=x]:=\frac{1}{Z}\exp\{\sum_{i,j}J_{i,j}x_{i}x_{j} + \sum_i (h_i+h) x_i\}
\]
and let $m_h$ denote the expected total magnetization for $\Pr_{h}$. Then, for any $\epsilon > 0$, we can find in constant time (depending only on $\epsilon$ and $\Delta$), an $\epsilon (n+\|\vec{J} \|_1)$ additive approximation to $m_h$, for some $h$ with $|h| < \eps$. 
\end{theorem}
\begin{proof}
%The case where $\| \vec{J}_1 \| = o(n)$ is trivial as in this case by Dobrushin uniqueness and the density condition. By the density condition $\| \vec{J}_{\infty} \| = o(1/n)$. This implies by Dobrushin uniqueness that we may assume WLOG that $J$ is identically zero with an error of $o(1)$ of magnetization per site. 
%Now the density condition on $h$ implies we can approximate the magnetization by a finite sample of single sites and computing the magnetization of each. 
%So WLOG we assume that $\| \vec{J}_1 \| = \Omega(n)$. 
It is well known that one can express the moments of spin systems in terms of derivatives of the log partition function. In particular, for the Ising model $\Pr[X=x]=\frac{1}{Z}\exp\{\sum_{i,j}J_{i,j}x_{i}x_{j}+\sum_{i}h_{i}x_{i}\}$,
consider the family of perturbed Ising models defined by $\Pr_{h}[X=x]=\frac{1}{Z_{h}}\exp\{\sum_{i,j}J_{i,j}x_{i}x_{j}+\sum_{i}(h_{i}+h)x_{i}\}$.
Then, for any $h_{0}$, we have 

\begin{align*}
\frac{\partial\log Z_{h}}{\partial h}(h_{0}) & =\frac{1}{Z_{h_{0}}}\frac{\partial}{\partial h}\left(\sum_{x\in\{\pm1\}^{n}}\exp\{\sum_{i,j}J_{i,j}x_{i}x_{j}+\sum_{i}(h_{i}+h)x_{i}\}\right)\\
 & =\sum_{x\in\{\pm1\}^{n}}\frac{1}{Z_{h_{0}}}\left(\exp\{\sum_{i,j}J_{i,j}x_{i}x_{j}+\sum_{i}(h_{i}+h_{0})x_{i}\}\right)\left(\sum_{i}x_{i}\right)\\
 & =\boldsymbol{E}_{h_{0}}[\sum_{i}x_{i}]
\end{align*}
where $\boldsymbol{E}_{h_{0}}$ denotes the expectation with respect
to the Ising distribution perturbed by $h_{0}$. In particular, $\frac{\partial\log Z_{h}}{\partial h}(0)$
equals the expected total magnetization of the Ising model we started
out with. Moreover, since by Jensen's inequality,
\begin{align*}
\frac{\partial^{2}\log Z_{h}}{\partial h^{2}}(h_{0}) & =\frac{\partial}{\partial h}|_{h=h_{0}}\sum_{x\in\{\pm1\}^{n}}\frac{1}{Z_{h_{0}}}\left(\exp\{\sum_{i,j}J_{i,j}x_{i}x_{j}+\sum_{i}(h_{i}+h_{0})x_{i}\}\right)\left(\sum_{i}x_{i}\right)\\
 & =\boldsymbol{E}_{h_{0}}[(\sum_{i}x_{i})^{2}]-(\boldsymbol{E}_{h_{0}}[\sum_{i}x_{i}])^{2}\\
 & \geq0
\end{align*}
we see that $\log Z$ is convex in $h$; in particular, for any $h_{0}\in\R$
and any $\delta>0$, we have 

\[
\frac{\log Z(h_{0})-\log Z(h_{0}-\delta)}{\delta}\leq\frac{\partial\log Z}{\partial h}(h_{0})\leq\frac{\log Z(h_{0}+\delta)-\log Z(h_{0})}{\delta} 
\]
Finally, 
\begin{itemize}
\item By the mean value theorem, the LHS /RHS  of the equation above are given by 
$\boldsymbol{E}_{h'}[\sum_{i}x_{i}]$ and $\boldsymbol{E}_{h''}[\sum_{i}x_{i}]$, where
$h_0 - \delta < h' < h_0 < h'' < h_0 + \delta$.
\item By setting $\epsilon = \delta^2$ in  \cref{thm-Delta-dense} and \cref{rmk-general-mrf}, we can evaluate the LHS and RHS up to the desired error in constant time 
%{\bf Need Theorem Name here}
\end{itemize}

\end{proof}

We remark that: 
\begin{itemize}
\item Unfortunately, it is impossible to approximate in constant time the magnetization at the specified value of the external fields. For example, consider an Ising model on $4 n$ vertices, where 
$J_{i,j} = C$ for some large $C$ if $i,j \leq 2n$ and $J_{i,j} = 0$ otherwise. 
Let $h_i = 1$ if $i \in [2n+1,3n]$ and $h_i = -1$ if $i \in [3n+1,4n]$.
We set all the other $h_i$ to $0$ except that we set $h_I = X$, where 
$I$ is uniformly chosen in $[1,2n]$ and $X$ is uniformly chosen in $\{0,\pm 1\}$. Note that this is a dense Ising model as per our definition. Note also that on the nodes $[1,2n]$ we have the Ising model on the complete graph with one (random) node having external field.

It is easy to see that if $X = 0$, the magnetization is $0$. 
The fact that $C$ is a large constant implies that conditioning on one vertex taking the value $\pm$ results in a dramatic change in magnetization on the vertices $[1,2n]$. In particular, the magnetization is of order $n$ if $X = +1$ and is of order $-n$ if $X = -1$. 
It thus follows that we need $\Omega(n)$ queries in order to determine the magnetization in this case. 
We note that this example corresponds to a phase transition -- in particular, for every $\epsilon > 0$, if 
$h' > \epsilon$ then $\boldsymbol{E}_{h'}[\sum_{i}x_{i}] = \Omega(n)$ for all values of $X$ and $I$. 
See (\cite{ellis2007entropy}) for general references for the Ising model on the complete graph.

\item The results for computing the magnetization readily extend to other models. 
For example, for Potts models, we can compute for each color the expected number of nodes of that color 
(up to error $\epsilon \| \vec{J} \|_1$ and for an $\epsilon$ close external field). 
Similarly, it is easy to check we can compute other statistics at this accuracy. For instance, for the Ising model, we can approximate $\boldsymbol{E}[\sum a_i x_i]$ if $n \eta \|a\|_{\infty} \leq \| a \|_1$ for some $\eta > 0$.

\end{itemize}

\subsection{Tightness of our results}\label{tightness-section}
As mentioned in the introduction, our results are \emph{qualitatively
tight}. We now make this precise by proving the following \emph{information-theoretic lower bound}
on the accuracy of constant-time algorithms for additively approximating $\log Z$: 
%any algorithm seeking to attain
%a $\epsilon \|\vec J\|_1$ w.h.p. additive approximation
%must probe $\Omega(1/\Delta\epsilon)$ many matrix entries.
\begin{theorem}\label{thm-qualitative-tightness}
Fix $\epsilon, \Delta \in (0,1/4)$. For any (possibly randomized) algorithm
$\mathcal A$ which probes at most $k := \frac{1}{8\epsilon \Delta}$ entries
of $J$ before returning an estimate to $\log Z$, there exists a $\Delta$-dense
input instance $J$ such that $\mathcal A$ makes error at least $\epsilon \|\vec J\|_1/4$ with probability at least $1/4$.
\iffalse
Let $k$ be an arbitrary positive integer, and
let $\mathcal A$ be any (possibly randomized) algorithm which probes
at most $k$ entries of $J$ before returning an estimate to $\log Z$. 
There exists an $\Omega(1/k)$-dense input instance $J$ 
where $\mathcal A$ makes error at least
$\|\vec J\|_1/4$ with probability at least $1/4$.
\fi
\end{theorem}
\begin{proof}
We prove the claim by reduction to a hypothesis testing problem. Specifically, we show that
there exist two different dense Ising models $J_M$ and $J'_M$ with log-partition functions that are at least $\epsilon\|\vec J'_M\|_1/2$-far apart (where $\|\vec J_M\| > \|\vec J'_M\|$) such that no algorithm which makes only $k$ probes can distinguish between the two with probability greater than $3/4$. 
%and which have log-partition functions that are at least $\epsilon\|\vec J'_M\|_1/2$-far apart where $\|\vec J_M\| > \|\vec J'_M\|$. 
This immediately implies that for any algorithm $\mathcal A$ to estimate $\log Z$ and for at least one of the two inputs, $\mathcal A$ must make error at least $\epsilon\|\vec J'_M\|_1/4$ with probability at least $1/4$ when given this input ---
otherwise we could use the output of $\mathcal A$ to distinguish the two models with better
than $3/4$ probability by checking which $\log Z$ the output is closer to.

\iffalse
Let $\epsilon<0.1$, $n\geq2^{\tilde{O}(1/\epsilon^{3})}$,
\fi
Let $n$ be an instance size to be taken sufficiently large, and consider two $\Delta$-dense ferromagnetic Ising models defined
as follows: 
\begin{itemize}
\item $J_{M}$, for which the underlying graph is the complete graph on
$n$ vertices, $\epsilon\Delta {n \choose 2}$ many of the edges are randomly selected
to have weight $\frac{M}{\Delta}$, and the remaining $(1-\epsilon \Delta){n \choose 2}$ many
 edges are assigned weight $M$. Note that since $||\vec{J}_{M}||_{\infty}=\frac{M}{\Delta}$
and $||\vec{J}_{M}||_{1}=\epsilon \Delta {n \choose 2}\frac{M}{\Delta}+(1-\epsilon \Delta){n \choose 2}M = (1 + \epsilon(1 - \Delta)) M {n \choose 2}$,
this model is indeed $\Delta$-dense for $n$ sufficiently large. 
\item $J'_{M}$, for which the underlying graph is the complete graph on
$n$ vertices and all edges have weight $M$. 
\end{itemize}
We denote the partition functions of these models by $Z_{M}$ and
$Z'_{M}$ respectively. It is easily seen that $\lim_{M\rightarrow\infty}\frac{\log Z_{M}}{M}=\lim_{M\rightarrow\infty}\frac{||\vec{J_{M}}||_{1}}{M}=(1 + \epsilon(1 - \Delta)){n \choose 2} \ge (1 + 3\epsilon/4){n \choose 2}$,
and that $\lim_{M\rightarrow\infty}\frac{\log Z'_{M}}{M}=\lim_{M\rightarrow\infty}\frac{||\vec{J'_{M}||_{1}}}{M}={n \choose 2}$.
Therefore, for $M$ sufficiently large, $\log Z_M \ge (1 + \epsilon/2)M{n \choose 2}$
and $\log Z'_M = M{n \choose 2}$, so $|\log Z_M - \log Z'_M| \ge (\epsilon/2) M{n \choose 2} = (\epsilon/2) \|\vec{J'}_M\|_1$.
\iffalse
Therefore, for $M$ sufficiently large, an $\epsilon||\vec{J_{M}}||_{1}$-additive
approximation to $\log Z_{M}$ lies between $(2\pm4\epsilon)M{n \choose 2}$
whereas an $\epsilon||\vec{J'_{M}}||_{1}$-additive approximation
to $\log Z'_{M}$ lies between $(1\pm2\epsilon)M{n\ choose 2}$. 
\fi

Now, we show that no algorithm $\mathcal A'$ can distinguish between
$J_M$ and $J'_M$ with probability greater than $3/4$ with only $k$ probes.
We fix a 50/50
split between $J_M$ and $J'_M$ on our input $J$ to algorithm $\mathcal A'$.
Since the randomized algorithm $\mathcal A'$ can be viewed as a mixture
over deterministic algorithms, %by convexity/Yao's minimax principle 
there must exist a deterministic algorithm $\mathcal A''$ with success probability
in distinguishing $J_M$ from $J'_M$ at least as large as $\mathcal A'$. Let $(u_1,v_1)$
be the first edge queried by $\mathcal A''$, let $(u_2,v_2)$ be the next edge queried
assuming $J_{u_1 v_1} = M$, and define $(u_3,v_3), \ldots, (u_k,v_k)$ similarly (without loss of generality,
the algorithm uses all $k$ of its available queries). Let $E$ be the event that
$J_{u_1,v_1}, \ldots, J_{u_k,v_k}$ are all equal to $M$. Event $E$ always
happens under $J_M$, and when the input is $J'_M$ we see $\Pr(E|J = J'_M) \ge 1 - k \frac{\epsilon \Delta n(n-1)/2}{n(n - 1)/2 - k} \ge 1 - 2k \epsilon \Delta$ for $n > 4k$. Thus the total variation distance between the observed distribution
under $J_M$ and $J'_M$ is at most $2k \epsilon \Delta$,
%Conditioned on $E$,
%the algorithm cannot distinguish $J_M$ and $J'_M$ with probability greater than $1/2$,
so by the Neyman-Pearson Lemma we know $\mathcal A''$ fails with probability at least $(1/2)(1 - 2k \epsilon \Delta)$.
Therefore for $k \le \frac{1}{4\epsilon\Delta}$ we see that $\mathcal A''$ fails with probability at least $1/4$
which proves the result.
%$\Pr(E) = \frac{{(1 - \epsilon) n(n + 1)/2 \choose k}}{{n(n + 1)/2 \choose k}} $
\iffalse
Now, consider any algorithm $A$ running in time $1/{100\epsilon}$. Since the expected number of ``heavy'' edges in $J_{M}$ seen by
$A$ is $1/100$, it follows from Markov's inequality
that with probability $\geq0.99$, such an algorithm cannot distinguish
between $J_{M}$ and $J'_{M}$. In particular, $A$ must fail to output
the desired additive approximation to the log partition function on
at least one of these models with probability $\geq0.49$.
\fi
\end{proof}
Note that on the input instances given in this theorem,
our algorithm will use at most $2^{\tilde{O}(1/\epsilon^{2}\Delta)}$
queries to output an $\epsilon \|\vec J\|_1$ approximation
for both models with probability $\geq0.99$. Finding the optimal dependence of the running
time on the parameters $\epsilon,\Delta$ remains an interesting open
problem. 
\subsection{Outline of the techniques}
To illustrate the main reason for the intractability of the (log) partition
function of an Ising model, we consider the ferromagnetic case where
$\Pr[X=x]=\frac{1}{Z}\exp\{\sum_{i,j}J_{i,j}x_{i}x_{j}\}$, $J_{i,j}\geq0$.
In this case, it is clear that a given \emph{magnetized} state $x$, 
where almost all of the spins are either $+1$ or $-1$, is more likely than a given \emph{unmagnetized }state $y$, where the spins are almost evenly split between $+1$ and $-1$. However, since the number of states with exactly $\alpha n$ spins equal to $+1$ is simply ${n \choose \alpha n}$, we see that the total number of strongly magnetized states is exponentially
smaller than the total number of unmagnetized states. Therefore, while
any given unmagnetized state is less likely than any given magnetized
state, it may very well be the case that the \emph{total} probability
of the system being in an unmagnetized state is greater than that
of the system being in a magnetized state.

In this paper, we present an approach to dealing with this tradeoff in dense Ising models (\cref{thm-Delta-dense}) based on the algorithmic regularity lemma of Frieze and Kannan (\cref{fk}). Roughly speaking, this lemma allows us to efficiently partition the underlying (weighted) graph into a small number of blocks in a manner such that ``cut-like'' quantities associated to the graph approximately depend only on the \emph{numbers} of edges between various blocks. In \cref{applying-reg-lemma}, we show that the log partition function fits into this framework. A similar
statement for MAX-CUT (which may be viewed as the limiting case of our model when only
antiferromagnetic interactions are allowed and $||\vec{J}||_{\infty}$ tends to infinity) was first obtained in \cite{frieze-kannan-old}.

The key point in the regularity lemma is that the number of blocks depends only on the desired quality of approximation, and \emph{not} of the size of the underlying graph. Together with the previous statement, this shows (\cref{z''-approx}) that one can approximately rewrite the sum computing the partition function in terms of only \emph{polynomially} many nonnegative summands, as opposed to the \emph{exponentially} many nonnegative summands we started out with. This provides a way to resolve the entropy-energy tradeoff -- the log of the largest summand of this much smaller sum approximates the log of the partition function  well (\cref{approx-sum-by-max}). 
\iffalse
We note here that this method of resolving the energy-entropy tradeoff, while being very simple,  is also faster than some state-of-the-art methods, and provides essentially the same guarantees in our setting (see the footnote in \cref{footnote}). \fnote{a bit misleading, calculating integrals isn't a way to resolve energy-entropy tradeoff, just a better way to compute entropy}
\fi

The preceding discussion reduces the problem of estimating the log partition function to an optimization problem, although a very different one from \cite{risteski-ising}. However, as stated, it is not a problem we can solve in constant time. In \cref{lemma:gamma-def}, we show how to ``granulate'' the parameters to reduce the problem to constant size,
and then our Algorithm~\ref{convex-partition} solves this problem efficiently via convex programming. The proofs of \cref{thm-mrf}, \cref{thm-mrf-nonconstant} and \cref{thm-ltr} follow a similar outline, with the application of \cref{fk} replaced by \cref{reg-fk-higher}, \cref{reg-alon-etal} and \cref{reg-ghar-trev} respectively.   
\begin{comment}
On the other hand, given such a decomposition of the graph, it is
readily seen that \emph{all} states which have approximately the same
number of $+1$ spins in each block contribute approximately the same
to the partition function. We emphasize here that \emph{where }the
$+1$ spins occur in each block does not matter \textendash{} only
their total number does. Moreover, since we are interested in only
approximating the (log) partition function anyway, even knowing the
fraction of $+1$ spins in each block with some constant precision
is already enough. Since there are only a constant number of blocks
to start with, we see that we now have a constant sized problem. 
\vnote{TO DO} 
\end{comment}

\section{Preliminaries}

In the case of the dense Ising model, our reduction to a problem that can be solved in constant time will
be based on the algorithmic weak regularity lemma of Frieze and Kannan \cite{frieze-kannan-matrix}. Before stating
it, we introduce some terminology. Throughout this section, we will
deal with $m\times n$ matrices whose entries we will index by $[m]\times[n]$,
where $[k]=\{1,\dots,k\}$. 
\begin{defn}
Given $S\subseteq[m]$, $T\subseteq[n]$ and $d\in\R$, we define
the $[m]\times[n]$ \emph{Cut Matrix }$C=CUT(S,T,d)$ by 
\[
C(i,j)=\begin{cases}
d & \text{if }(i,j)\in S\times T\\
0 & \text{otherwise}
\end{cases}
\]
\end{defn}

\begin{defn}
A \emph{Cut Decomposition }expresses a matrix $J$ as 
\[
J=D^{(1)}+\dots+D^{(s)}+W
\]

where $D^{(i)}=CUT(R_{i},C_{i},d_{i})$ for all $t=1,\dots,s$. 
We say that such a cut decomposition has \emph{width }$s$\emph{,
coefficient length $(d_{1}^{2}+\dots+d_{s}^{2})^{1/2}$ }and \emph{error
$||W||_{\infty\mapsto1}$}.
\end{defn}

We are now ready to state the algorithmic weak regularity lemma of Frieze and Kannan. 
\begin{theorem}
\label{fk}
\cite{frieze-kannan-matrix}
Let $J$ be an arbitrary real matrix, and let $\epsilon,\delta>0$.
Then, in time $2^{\tilde{O}(1/\epsilon^{2})}/\delta^{2}$, we can,
with probability $\geq1-\delta$, find a cut decomposition of width
$O(\epsilon^{-2})$, coefficient length at most $\sqrt{27}||\vec{J}||_{2}/\sqrt{mn}$
and error at most $4\epsilon\sqrt{mn}||\vec{J}||_{2}$. 
\end{theorem}

\section{Reducing the number of summands using weak regularity}%Pointwise approximation
The next lemma shows that for the purpose of additively approximating
the log partition function, we may as well work with the matrix $D^{(1)}+\dots+D^{(s)}$. We define $Z' := \sum_x \exp(x^T(D^{(1)} + \cdots + D^{(s)}) x)$. 

\begin{lemma}\label{applying-reg-lemma} Let $J$ be the matrix of interaction strengths of a $\Delta$-dense
Ising model. For $\epsilon,\delta>0$, let $J=D^{(1)}+\dots+D^{(s)}+W$
be a cut decomposition of $J$ as in \cref{fk}.Then, for $Z$ and $Z'$ as above, we have 
$|\log Z-\log Z'|\leq\frac{4\epsilon}{\sqrt{\Delta}}||\vec{J}||_{1}$.

\end{lemma}
\begin{proof}
Note that for any $x\in\{\pm1\}^{n}$, we have 
\begin{align*}
|\sum_{i,j}J_{i,j}x_{i}x_{j}-\sum_{i,j}(D^{(1)}+\dots+D^{(s)})_{i,j}x_{i}x_{j}| & =|\sum_{i}(\sum_{j}W_{i,j}x_{j})x_{i}| \leq|\sum_{i}|\sum_{j}W_{i,j}x_{j}|\\
 & \leq||W||_{\infty\mapsto1} \leq4\epsilon n||\vec{J}||_{2}
\end{align*}
Therefore, for any $x\in\{\pm1\}^{n}$, we have 
%\[
%\exp(-4\epsilon n||J||_{F}+\sum_{i,j}
%(D^{(1)}+\dots+D^{(s)})_{i,j}x_{i}x_{j})\leq\exp(\sum_{i,j}J_{i,j}x_{i}x_{j})\leq\exp(4\epsilon %n||J||_{F}+\sum_{i,j}(D^{(1)}+\dots+D^{(s)})_{i,j}x_{i}x_{j})
%\]
\[
\exp(\sum_{i,j}J_{i,j}x_{i}x_{j}) \in \left[ \exp\left(\sum_{i,j}(D^{(1)}+\dots+D^{(s)})_{i,j}x_{i}x_{j}) \pm 4\epsilon n||\vec{J}||_{2}\right) \right].
\]
Taking first the sum of these inequalities over all $x\in\{\pm1\}^{n}$
and then the log, we get 
%\[
%\log\sum_{x\in\{\pm1\}^{n}}\exp(x^{T}(D^{(1)}+\dots+D^{(s)})x)-4\epsilon n||J||_{F}
%\leq\log Z\leq\log\sum_{x\in\%{\pm1\}^{n}}\exp(x^{T}(D^{(1)}+\dots+D^{(s)})x)+4\epsilon n||J||_{F}
%\]
\[
\log Z \in \left[ \log\sum_{x\in\{\pm1\}^{n}}\exp \left(x^{T}(D^{(1)}+\dots+D^{(s)})x \right) \pm 4\epsilon n||\vec{J}||_{2} \right]
\]
Finally, noting that $||\vec{J}||_{2}^{2}=\sum_{i,j}|J_{i,j}|^{2}\leq||\vec{J}||_{\infty}||\vec{J}||_{1}\leq\frac{||\vec{J}||_{1}^{2}}{\Delta n^{2}}$, and the proof follows.  
\end{proof}

%So far, we have seen how the problem of estimating the log partition function can be reduced to the problem of approximating $\log Z'$, where
%$Z' := \sum_x \exp(x^T(D^{(1)} + \cdots + D^{(s)}) x)$
%%within additive error $\epsilon J_T$ for some $\epsilon > 0$. 
%We will now show how to find this approximation using Algorithm~\ref{convex-partition}. %First, we need a few definitions.

Recall that for each $i\in[s]$, $D^{(i)} = CUT(R_i,C_i,d_i)$ is a cut
matrix between vertex sets $R_i$ and $C_i$.
We pass to a common refinement of all of the $R_i$'s and $C_i$'s. This
gives at most $r \le 2^{2s}$ disjoint sets $V_1, \ldots, V_r$ such that
that every one of the $R_i$'s and $C_i$'s is a union of these ``atoms'' $V_a$. 
%Let $v_a := |V_a|/n$ be the relative size of the set $V_a$. 
In terms of these $V_a$'s, we can define another approximation $Z''$ to the partition function: 
\[
Z'':=\sum_{y}\exp\left(\sum_{i=1}^{s}r'(y)_{i}c'(y)_{i}d_{i}+\sum_{a=1}^{r}|V_{a}|H(y_{a}/|V_{a}|)\right)
\]
where $y$ ranges over all elements of $[|V_{1}|]\times\dots\times[|V_{r}|]$,
$r'(y)_{i}:=\sum_{a\colon V_{a}\subseteq R_{i}}(2y_{a}-|V_{a}|)$,
$c'(y)_{i}:=\sum_{a\colon V_{a}\subseteq C_{i}}(2y_{a}-|V_{a}|)$
and $H(p):=-p\log p-(1-p)\log(1-p)$ is the binary entropy function. 
Note that $r'(y)_{i}$ represents the net spin in $R_{i}$, and similarly for $c'(y)_{i}$. The next lemma shows that for approximating $\log{Z'}$, it suffices to obtain an approximation to $\log{Z''}$. Combined with \cref{applying-reg-lemma}, we therefore see that approximating $\log{Z''}$ is sufficient for approximating $\log{Z}$.  

\begin{lemma}\label{z''-approx}
For $Z'$ and $Z''$ as above, we have $|\log Z' - \log Z''| = O(2^{2s} \log n)$.
\end{lemma}
\begin{proof}
First, observe that $x^T D^{(i)} x = d_i \sum_{a \in R_i, b \in C_i} x_a x_b = d_i (\sum_{a \in R_i} x_a)(\sum_{b \in C_i} x_b)$. 
Thus, letting $r'_i(x) := \sum_{a \in R_i} x_a$ and $c'_i(x) := \sum_{b \in C_i} x_b$, we see that that
\[ Z' = \sum_x \exp(x^T(D^{(1)} + \cdots + D^{(s)}) x) = \sum_{x} \exp\left(\sum_i r'_i(x) c'_i(x) d_i\right). \]
Re-expressing the summation in terms of the possible values that $r'_i(x)$ and $c'_i(x)$ can take, we get that
\[ Z' = \sum_{r',c'} \exp\left(\sum_i r'_i c'_i d_i\right) \sum_{\substack{x \\ r'(x) = r' \\ c'(x) = c'}} 1. \]
where $r'$ ranges over all elements of $[|R_{1}|]\times\dots\times[|R_{s}|]$,
$r':=(r'_{1},\dots,r'_{s})$, $r'(x):=(r'_{1}(x),\dots,r'_{s}(x))$
and similarly for $c'$ and $c'(x)$. 
Next, since 
\[ \sum_{\substack{x \in \{\pm1\}^{n} \\ r'(x) = r'\\ c'(x) = c'}} 1 = \sum_{\substack{y \in [|V_{1}|]\times\dots\times[|V_{r}|]\\ r'(y) = r'\\ c'(y) = c'}} \prod_a {|V_a| \choose y_a} \]
%where $y_a$ denotes the number of $+$-spins in atom $V_a$,
it follows that  $Z' = \sum_{y} \exp\left(\sum_i r'(y)_i c'(y)_i d_i\right)  \prod_a {|V_a| \choose y_a}$. Finally, we apply Stirling's formula
\[ \log {n \choose \alpha n} = n H(\alpha) + O(\log n). \] to get the desired result.
\end{proof}

\section{Approximating the reduced sum in constant time using convex programming}
%{Finding the Bulk Contribution to the Partition Function}
So far, we have seen how the problem of estimating the log partition function can be reduced to the problem of approximating $\log Z''$. 
%We will now show how to carry out this approximation using Algorithm~\ref{convex-partition}. Throughout, we will use $v_a$ to denote $|V_a|/n$ i.e. the relative size of the set $V_a$. 
The next simple lemma reduces the estimation of $\log{Z''}$ to an optimization problem.
\begin{lemma}\label{approx-sum-by-max}
Let $y^{\ast}:=\arg\max_{y}\exp\left(\sum_{i}r'_i(y)c'_i(y)d_{i}+\sum_{a}|V_{a}|H(y_{a}/|V_{a}|)\right)$.
Then, 
\[
\left|\log Z''-\left(\sum_{i}r'_i(y^{\ast})c'_i(y^{\ast})d_{i}+\sum_{a}|V_{a}|H(y_{a}^{\ast}/|V_{a}|)\right)\right|\leq2^{2s}\log n.
\]
\end{lemma}
\begin{proof}
This follows immediately by noting that $Z''$, which is a sum of at most $n^{2^{2s}}$ many nonnegative summands, is at least as large as its largest summand, and no larger than $n^{2^{2s}}$ times its largest summand. 
\end{proof}
The following lemma shows that for estimating the contribution of the term corresponding to some vector $y$, it suffices to know the components of $y$ up to some constant precision. This reduces the optimization problem to one of constant size.  
\begin{lemma}\label{lemma:gamma-def}
Let $J$ denote the matrix of interaction strengths of a $\Delta$-dense Ising model, and let $J = D^{1}+\dots+D^{s} + W$ denote a cut decomposition as in \cref{fk}. Then, given $r_{i},r'_{i},c_{i},c'_{i}$ for $i\in[s]$ and some 
$\gamma \le \frac{\epsilon \sqrt{\Delta}}{4 \sqrt{27} s}$ such that $r_i,c_i,r'_i,c'_i \le n$, 
$|r_i - r'_i| \le \gamma n$ and $|c_i - c'_i| \le \gamma n$ 
for all $i$, we get that 
$\sum_i d_i|r'_i c'_i - r_i c_i| \le \epsilon ||\vec{J}||_{1}/2$. 
\end{lemma}
\begin{proof}
From \cref{fk}, we know that for all $i\in[s]$,  
%\begin{equation}\label{dtbound}
$|d_i| \le \frac{\sqrt{27} ||\vec{J}||_{1}}{n^2 \sqrt{\Delta}}$.
%\end{equation}
\iffalse
so using that $|R_t| \le n, |C_t| \le n$ we see that if we estimate
the number of $+$ and $-$ in $R_t,C_t$ within an additive error of $\gamma n$,
\fi
Since $ |r'_i c'_i - r_i c_i| \le |c'_i||r'_i - r_i| + |r_i||c'_i - c_i| \le 2\gamma n^2$, it follows that
\begin{align*}
\sum_i d_i|r'_i c'_i - r_i c_i|
\le \sum_i d_i 2\gamma n^2
\le \frac{\sqrt{27} ||\vec{J}||_{1} }{\sqrt{\Delta}} 2\gamma s
\le %\frac{2 \sqrt{27} s}{\sqrt{\Delta}} \gamma J_T
\frac{\eps \||\vec{J}||_{1}}{2}
\end{align*} which completes the proof. 
%therefore taking
%\[ \gamma \le \frac{\epsilon \sqrt{\Delta}}{4 \sqrt{27} s} \]
%we have the desired result.
\end{proof}

We now
present Algorithm~\ref{convex-partition}
which approximates the log of the largest summand in $Z''$ by iterating over all possible proportions of up-spins in each block $R_i, C_i$. For each proportion, it 
approximately solves the following convex program $\mathcal{C}_{\overline{r},\overline{c}}$ (where $\gamma$ is as before, $\lambda$ is a parameter to be specified, and $v_a := |V_a|/n$):
\begin{alignat*}{4}
&\max\quad &\sum_a &v_a H(z_a/v_a)\\
&\ s.t.\quad&0 &\le z_a &&\le v_a & \\
&&\overline{r_t} &\le \sum_{a : V_a \subset R_t} z_a &&\le \overline{r_t} + \gamma  \\
&&\overline{c_t} &\le \sum_{a : V_a \subset C_t} z_a &&\le \overline{c_t} + \gamma.
\end{alignat*}
The infeasibility of this program means that our ``guess'' for the proportion of up-spins in the various $R_i,C_i$ is not combinatorially possible (recall that there may be significant overlap between the various $R_i,C_i$. On the other hand, if this program is feasible, then we actually obtain an approximate maximum entropy configuration with roughly the prescribed number of up-spins in $R_i,C_i$, $i\in[n]$. Calculating the contribution to $Z''$ of such a configuration, and maximizing this contribution over all iterations gives a good approximation\footnote{A slightly
better approximation to $\log Z''$ can be found by 
%instead of solving a convex program to find the max-entropy contribution of each region, instead 
estimating the discrete sum over each region by an integral (using Lemma~\ref{lemma:entropy-rounding} to bound the error), and approximating this integral via the hit-and-run random walk of \cite{lovasz2006fast}. However, this algorithm
is slower than Algorithm~\ref{convex-partition}, and it turns out the gain in accuracy is negligible for our application.}\label{footnote}
to the maximum summand of $Z''$, and hence to $Z''$. This is the content of the next lemma. 
\begin{algorithm}
\caption{Convex programming method to estimate log partition unction}
\label{convex-partition}
\begin{algorithmic}
\State Let $S = \{0, \gamma, 2\gamma, \ldots, \lfloor 1/\gamma \rfloor \gamma \}$.
\State $M \gets 0$.
\For{$\overline{r} \in S^{s}, \overline{c} \in S^{s}$}
\State Either find a $\lambda$-approximate solution to, or verify infeasibility of, $\mathcal{C}_{\overline{r},\overline{c}}$. (e.g. by ellipsoid method \cite{gls}) 
\If{$\mathcal{C}_{\overline{r},\overline{c}}$ is feasible}
\State Let $H_{\overline r,\overline c}$ denote the %$\epsilon ||\vec{J}||_{1}$
$\lambda$-approximate maximum value of $\mathcal{C}_{\overline{r},\overline{c}}$.
\State Let $r'_i = 2n \overline{r_i} - |R_i|$ and $c'_i = 2n \overline{c_i} - |C_i|$.
\State Let $M_{\overline r,\overline c} = \exp\left(\sum_i r'_i c'_i d_i +  n H_{\overline{r},\overline{c}}\right)$.
\State $M \gets \max(M, M_{\overline r,\overline c})$.
\EndIf
\EndFor
\State \Return $\log{M}$
\end{algorithmic}
\end{algorithm}

\begin{comment}
We now explain how to choose the rounding parameter $\gamma$.
\begin{lemma}\label{lemma:gamma-def}
Suppose $|r_i - r'_i| \le \gamma n$ and $|c_i - c'_i| \le \gamma n$,
and $r_i,c_i,r'_i,c'_i \le n$
for all $i$, and
\[ \gamma \le \frac{\epsilon \sqrt{\Delta}}{4 \sqrt{27} s} \]
Then
\[ \sum_i d_i|r'_i c'_i - r_i c_i| \le \epsilon J_T/2. \]
\end{lemma}
\begin{proof}
Note that from the above we know 
%\begin{equation}\label{dtbound}
\[ |d_t| \le \frac{\sqrt{27} J_T}{n^2 \sqrt{\Delta}} \]
%\end{equation}
\iffalse
so using that $|R_t| \le n, |C_t| \le n$ we see that if we estimate
the number of $+$ and $-$ in $R_t,C_t$ within an additive error of $\gamma n$,
\fi
so using that
\[ |r'_i c'_i - r_i c_i| \le |c'_i||r'_i - r_i| + |r_i||c'_i - c_i| \le 2\gamma n^2 \]
we see
\begin{align*}
\sum_i d_i|r'_i c'_i - r_i c_i|
\le \sum_i d_i 2\gamma n^2
\le \sum_i \frac{\sqrt{27} J_T}{\sqrt{\Delta}} 2\gamma
\le \frac{2 \sqrt{27} s}{\sqrt{\Delta}} \gamma J_T
\end{align*}
therefore taking
\[ \gamma \le \frac{\epsilon \sqrt{\Delta}}{4 \sqrt{27} s} \]
we have the desired result.
\end{proof}
\end{comment}

\begin{lemma}\label{convex-partition-correctness}
The output of Algorithm~\ref{convex-partition} is an $\epsilon ||\vec{J}||_{1} + \lambda n + O(2^{2s}\log n)$
additive approximation to $\log Z''$ when $n \geq \frac{4\sqrt{27}s2^{2s}}{\sqrt{\Delta} \epsilon}$.
\end{lemma}
\begin{proof}
\begin{comment}
Recall that
\[ Z'' = \sum_{y} \exp\left(\sum_i r'(y)_i c'(y)_i d_i + \sum_a |V_a| H(y_a/V_a)\right) \]
Observe that if 
$y$ is an arbitrary term in the sum and $y^*$ corresponds
to the largest term in the sum,
we know that
\begin{align}
Z'' &\ge\exp\left(\sum_i r'(y)_i c'(y)_i d_i + \sum_a |V_a| H(y_a/V_a)\right), \label{z''-below} \\
Z'' &\le n^{2^{2s}} \exp\left(\sum_i r'(y^*)_i c'(y^*)_i d_i + \sum_a |V_a|, \label{z''-above} H(y^*_a/V_a)\right).
\end{align}
We will use these bounds to relate $Z''$ to the output $M$ of Algorithm~\ref{convex-partition}.
\end{comment}
By \cref{approx-sum-by-max}, it suffices to prove the claim with $\log{Z''}$ replaced by 
%Let $y^{*}$ be as in \cref{approx-sum-by-max} and recall that we reduced approximating $\log Z''$ 
%to estimating
\[ \sum_{i}r'_i(y^{\ast})c'_i(y^{\ast})d_{i}+\sum_{a}|V_{a}|H(y_{a}^{\ast}/|V_{a}|). \]
where $y^{*}$ is defined as in the statement of \cref{approx-sum-by-max}.
Let $\overline{r},\overline{c}$ correspond to the ``cell'' which $y^*/n$ falls into,
i.e. the values of $\overline{r},\overline{c}$ such that for the corresponding
$r',c'$, we have $r' \le r'(y^*)$,$c' \le c'(y^*)$, $|r' - r'(y^*)| \le \gamma n$
and $|c' - c'(y^*)| \le \gamma n$. Since $H_{\overline{r},\overline{c}}$ 
is the result of solving $\mathcal{C}_{\overline{r},\overline{c}}$ up to $\lambda$
additive error and since $y^*/n$ is in the feasible region of the convex program by definition, it follows that
\[ n H_{\overline{r},\overline{c}} \ge n \sup_{z} \sum_a v_a H(z_a/v_a) - \lambda n \ge \sum_a |V_a| H(y^*_a/|V_a|) - \lambda n \]
where $z$ ranges over the feasible region for $\mathcal{C}_{\overline{r},\overline{c}}$.
Combining this with Lemma~\ref{lemma:gamma-def} gives 
\begin{align*}
\sum_i &r'_i(y^*) c'_i(y^*) d_i + \sum_a |V_a| H(y^*_a/|V_a|) - \log M_{\overline{r},\overline{c}} \\
&\le \sum_i d_i (r'_i(y^*)_i c'_i(y^*) - r'_i c'_i) + \sum_a |V_a| H(y^*_a/|V_a|) - n H_{\overline{r},\overline{c}} \\
&\le \sum_i d_i |r'(y^*) c'(y^*) - r'_i c'_i| + \lambda n \le \epsilon ||\vec{J}||_{1}/2 + \lambda n.
\end{align*}
%therefore by \eqref{z''-above}
%\[ \log M \ge \sum_i r'(y^*)_i c'(y^*)_i d_i + \sum_a |V_a| H(y^*_a/V_a) - \epsilon J_T/2 - n\lambda \ge \log Z'' - 2^{2s} \log n - \epsilon J_T/2 - \lambda n. \]
On the other hand, we know that $M = M_{\overline r,\overline c}$ for some
$\overline{r},\overline{c}$. Let $z$ denote the point at which the $\lambda$-approximate optimum to 
the convex program $\mathcal{C}_{\overline r, \overline c}$ was attained. Define $y_a$ to be $\lceil z_a n \rceil$ if $z_a n \le |V_a|/2$ and to be
$\lfloor z_a n \rfloor$ otherwise. Similar to the inequality above, we get 
\begin{align*}
\log M_{\overline{r},\overline{c}}-&\sum_{i}r'_i(y^{\ast})c'_i(y^{\ast})d_{i}+\sum_{a}|V_{a}|H(y_{a}^{\ast}/|V_{a}|)\\ 
&\leq\log M_{\overline{r},\overline{c}}-\sum_{i}r'_i(y)c'_i(y)d_{i}-\sum_{a}|V_{a}|H(y_{a}/|V_{a}|)\\
%% Unecessary, cut line to shorten display
% & \leq\left(nH_{\overline{r},\overline{c}}-\sum_{a}|V_{a}|H(y_{a}/|V_{a}|)\right)+\left(\sum_{i}r'_{i}c'_{i}d_i-\sum_{i}r'_i(y)c'_i(y)d_{i}\right)\\
 & =\sum_{a}\{|V_{a}|H(z_{a}n/|V_{a}|)-|V_{a}|H(y_{a}/|V_{a}|\}+\sum_{i}r'_{i}c'_{i}d_{i}-\sum_{i}r'_i(y)c'_i(y)d_{i}\\
%% Unecessary, cut line to shorten display
% & \leq2^{2s}\log5+\left|\sum_{i}r'_{i}c'_{i}d_{i}-\sum_{i}r'_i(y)c'_i(y)d_{i}\right|\\
 & \leq2^{2s}\log5+\left|\sum_{i}r'_{i}c'_{i}d_{i}-\sum_{i}r'_i(zn)c'_i(zn)d_{i}\right|+\left|\sum_{i}r'_{i}(zn)c'_{i}(zn)d_{i}-\sum_{i}r'_i(y)c'_i(y)d_{i}\right|\\
 & \leq2^{2s}\log5+\frac{\epsilon||\vec{J}||_{1}}{2}+\left|\sum_{i}r'_{i}(zn)c'_{i}(zn)d_{i}-\sum_{i}r'_i(y)c'_i(y)d_{i}\right|
\end{align*}
where $r'$ and $c'$ are as defined in Algorithm~\ref{convex-partition} in terms of $\overline{r}$ and $\overline{c}$,
in the second inequality, we have used \cref{lemma:entropy-rounding} and the triangle inequality, and in the
last line, we have used \cref{lemma:gamma-def}.
To bound the last term, since $|y_{a}-z_{a}n|\leq1$ for all $a$ by definition,
and since $|d_{i}|\leq\frac{\sqrt{27}||\vec{J}||_{1}}{n^{2}\sqrt{\Delta}}$
by \cref{fk}, we get 
\begin{align*}
|\sum_{i}r'_{i}(zn)c'_{i}(zn)d_{i}-\sum_{i}r'(y)_{i}c'(y)_{i}d_{i}| & \leq\sum_{i}|d_{i}||r'_{i}(y)||c'_{i}(y)-c'_{i}(zn)|+\sum_{i}|d_{i}||c'_{i}(zn)||r'_{i}(y)-r'_{i}(zn)|\\
 & \leq\sum_{i}|d_{i}||R_{i}||\sum_{a\colon V_{a}\subseteq C_{i}}2|+\sum_{i}|d_{i}||C_{i}||\sum_{a\colon V_{a}\subseteq C_{i}}2|\\
 & \leq2\sum_{i}|d_{i}|n2^{2s}\leq\frac{\sqrt{27}||\vec{J}||_{1}2^{2s+1}s}{n\sqrt{\Delta}}\\
 & \leq\frac{\epsilon||\vec{J}||_{1}}{2}
\end{align*}
provided that $n\geq\frac{4s\sqrt{27}2^{2s}}{\sqrt{\Delta}\epsilon}$,
which finishes the proof.

\begin{comment}
By Lemma~\ref{lemma:entropy-rounding}
we know that % $|V_a| |H(y_a/nv_a) - H(z_a)|\le \log 5$ so 
\[ \sum_a |V_a| (H(y_a/|V_a|) - H(z_a)) \le 2^{2s} \log 5. \]
Also we have from our bound on $|d_i|$ that
\begin{align*}
\sum_i |d_i| |r'(y) c'(y) - r'(zn)c'(zn)| 
&\le \sum_i |d_i| (|r'(y)||c'(y) - c'(zn)| + |c'(zn)||r'(y) - r'(zn)|) \\
&\le \sum_i 4|d_i|n \le \frac{4\sqrt{27} s J_T}{n \sqrt{\Delta}}
\end{align*}
Similar to above we see by Lemma~\ref{lemma:gamma-def} our bound on $|d_i|$ that
\begin{align*}
\sum_i d_i|r'(y) c'(y) - r' c'| 
&\le \sum_i d_i|r'(zn) c'(zn) - r' c'| + \sum_i d_i |r'(y) c'(y) - r'(zn)c'(zn)| \\
&\le \epsilon J_T/2 + \frac{4\sqrt{27} s J_T}{n \sqrt{\Delta}}
\end{align*}
and so by \eqref{z''-below}
\begin{align*}
\log M
&\le \sum_i r'(y)_i c'(y)_i d_i + \sum_a |V_a| H(y_a/V_a) + \epsilon J_T/2 + \frac{4\sqrt{27} s J_T}{n \sqrt{\Delta}} + O(2^{2s}) \\
&\le \log Z'' + \epsilon J_T/2 + \frac{4\sqrt{27} s J_T}{n \sqrt{\Delta}} + O(2^{2s})
\end{align*}

Combining our bounds on $\log M$ we find that
\[ |\log M - \log Z''| \le \epsilon J_T/2 + \frac{4\sqrt{27} s J_T}{n \sqrt{\Delta}}| + O(2^{2s}\log n) + \lambda n. \]
Finally if $n \ge \frac{8 \sqrt{27} s}{\sqrt{\Delta} \epsilon}$ then $\frac{4\sqrt{27} s J_T}{n \sqrt{\Delta}}| \le \epsilon J_T/2$
so we get our desired result.
%and using that $|J_T| = \Omega(n)$ gives the claimed result.
\end{comment}
\end{proof}
\begin{lemma}\label{lemma:entropy-rounding}
Let $n$ be a natural number, and let $0 \le z \le n$. Let $y = \lceil z \rceil$ if $z \le n/2$
and $y = \lfloor z \rfloor$ if $z > n/2$. Then,
$nH(y/n) \ge nH(z/n) - \log(5)$.
\end{lemma}
\begin{proof}
Observe that if both $y,z \le n/2$ or both $y,z \ge n/2$ then
this is immediate, since entropy is concave and maximized at $1/2$, thereby implying that $H(y/n) > H(z/n)$. In the remaining case, $y/n$ and $z/n$ are on opposite
sides of $1/2$. By the definition of $y$, we know that $|y/n - z/n| \le 1/n$, and hence, by the Mean Value Theorem, $n|H(y/n) - H(z/n)| = n|H'(x/n)(y/n - z/n)| \le |H'(x/n)|$ for some $x$ with $|x/n - 1/2| \le 1/n$.  Thus, if $n \ge 3$ we have $x/n \in [1/6,5/6]$ so that
$|H'(x/n)| = |\log((x/n)/(1 - x/n))| \le \log 5$.
If $n \le 2$, then we already have $nH(z/n) - nH(y/n) \le nH(z/n) \le 2\log2 < \log5$. 
\end{proof}

\section{Putting it all Together}
\iffalse
State the actual theorem. Note that dependence on success probability
$\delta$ in approximate regularity lemma is bad; since we only
need to output a number, should use some standard median trick where
we run the whole algorithm many times to
``boost'' our success probability and get only $\log(1/\delta)$ type dependence.
\fi
We are now ready to prove \cref{thm-Delta-dense}. 
%% Resolved, see end of 2.1
%\fnote{As Vishesh noted one must use sampling to estimate the proportions $v_a$, should we briefly mention this? Frieze-Kannan's max-cut alg just leaves it implicit.}
\begin{comment}
\begin{theorem}\label{thm-Delta-dense-full}
Suppose we have a $\Delta$-dense Ising model on $n$ nodes,
with $n = 2^{\Omega(1/\epsilon^2 \Delta)}$.
%and  $||\vec{J}||_{\infty} \ge 1/n$. 
There is an algorithm to compute an $\epsilon (\|\vec J\|_1 + n)$ additive approximation
to $\log Z$ with success probability at least $1 - \delta$, which runs
in time $2^{\tilde O(1/\epsilon^2\Delta)}\log(1/\delta)$.
\end{theorem}
\end{comment}
\begin{proof}(\cref{thm-Delta-dense})
First, we show how to get such an algorithm with constant success probability (i.e. $7/8$). We will then boost it to the desired probability of success using the standard median technique. 
The constant success algorithm is just 
Algorithm~\ref{convex-partition} applied to the weakly regular partition generated by Theorem~\ref{fk} with regularity parameter some $\epsilon' = O(\epsilon \sqrt{\Delta})$.

For the correctness of this algorithm, note that 
%by Lemma~\ref{applying-reg-lemma} (using $\epsilon' = O(\epsilon \sqrt{\Delta})$), Lemma~\ref{z''-approx}, and Lemma~\ref{convex-partition-correctness} we see that 
if $M$ is the output of Algorithm~\ref{convex-partition}, then
\begin{align*}
|\log M - \log Z| 
&\le |\log Z - \log Z'| + |\log Z' - \log Z''| + |\log M - \log Z''| \\
&\le \epsilon \|\vec{J}\|_1/4 + O(2^{2s} \log n) + \epsilon\|\vec{J}\|_1/4 + \lambda n + O(2^{2s}).
\end{align*}
where the first bound is by \cref{applying-reg-lemma}, the second bound is by \cref{z''-approx}, and the last bound is by \cref{convex-partition-correctness}. 
%Now, observe that under our assumptions, $\|\vec{J}\|_1 \ge n^2 \Delta \|\vec{J}\|_{\infty} \ge \Delta n$. Therefore, our assumption that $n$ is sufficiently large gives
%so with our requirement on $n$ being sufficiently large we have
\[
\epsilon \|\vec{J}\|_1/4 + O(2^{2s} \log n) + \epsilon\|\vec{J}\|_1/4 + \lambda n
\le \epsilon (\|\vec{J}\|_1 + n)
\]
when we take $\lambda = \epsilon/2$ and $n$ sufficiently large.
The runtime of computing a weakly regular partition
with success probability $7/8$ in Theorem~\ref{fk} is $2^{\tilde{O}(1/\epsilon^2 \Delta)}$. The runtime
of Algorithm~\ref{convex-partition} is $2^{\tilde{O}(1/\epsilon^2\Delta)}$ as there are
$2^{\tilde{O}(1/\epsilon^2\Delta)}$ many points in $|S|$, and solving each convex program $\mathcal{C}_{\overline{r},\overline{c}}$ up to $\lambda$-error (as well as deciding feasibility)
can be done in time $poly(2^{\tilde O(1/\epsilon^2\Delta)}, \log(1/\lambda), \log(1/\gamma)) = 2^{\tilde O(1/\epsilon^2\Delta)}$ by
standard guarantees for the ellipsoid method \cite{gls}.

Finally, to boost our success probability to $1 - \delta$, we use the standard median trick, which is to repeat the algorithm $O(\log(1/\delta))$ times independently and take the median of these outputs -- that this works is easily proven by the standard Chernoff bounds.
\end{proof}

\section{Extensions}
%\fnote{above is a binary MRF, should we allow general alphabets? or is there
%an easy reduction in the dense setting. we could also if we wanted drop the uniformity assumption
%and apply regularity to each tensor separately...}
\subsection{General Markov random fields}
\label{mrf}
For simplicity, we will restrict ourselves to %$k$-uniform 
Markov random fields over a binary alphabet. 
%with otherwise arbitrary interactions
It is readily seen that the same techniques extend to 
%non-uniform 
Markov random fields over general finite alphabets as well.  
\begin{defn} A \emph{binary Markov random field of order $K$} is a probability distribution on the discrete cube $\{\pm1\}^n$ of the form 
\[
\Pr[X=x]=\frac{1}{Z}\exp(\sum_{k=1}^{K}\sum_{i_{1},\dots,i_{k}=1}^{n}J_{i_{1}\dots i_{k}}^{(k)}x_{i_{1}}\dots x_{i_{k}})
\]
where for each $k$, $\{J^{(k)}_{i_{1},\dots,i_{k}}\}_{i_{1},\dots,i_{k}\in[n]}$ is an
arbitrary collection of real numbers. Note that the sum $\sum_{i_{1},\dots,i_{k}}J^{(k)}_{i_{1},\dots,i_{k}}x_{i_{1}}\dots x_{i_{k}}$
may also be viewed as $J^{(k)}(x,\dots,x)$, where $J^{(k)}$ is a $k$-multilinear
map or $k$-tensor with $J^{(k)}(e_{i_{1}},\dots,e_{i_{k}})=J^{(k)}_{i_{1},\dots,i_{k}}$. Here, $e_{j}$ denotes the $j^{th}$ standard basis vector of $\R^{n}$
and we view $\{\pm1\}^{n}$ as a subset of $\R^{n}$. The normalizing
constant $Z=\sum_{x\in\{\pm1\}^{n}}\exp(\sum_{k=1}^{K}J^{(k)}(x,\dots,x))$ is called
the \emph{partition function }of the Markov random field. Moreover, if all entries of all $J^{(i)}$, $i\in[K]\backslash{k}$ vanish, then we say that the Markov random field is $k$-uniform.  
\end{defn}
\begin{defn}
A binary Markov random field of order $K$ is \emph{$\Delta$-dense} if $\Delta||\vec{J^{(k)}}||_{\infty}\leq\frac{||\vec{J^{(k)}}||_{1}}{n^{k}}$ for all $k$.
\end{defn}
As indicated earlier, our approach for general Markov random fields
mirrors our approach for the Ising model. This is essentially because
one has a similar algorithmic regularity lemma for general $k$-dimensional
matrices. 

\begin{defn}
A $k$-dimensional matrix $M$ on $X_{1}\times\dots\times X_{k}$
is a map $M\colon X_{1}\times\dots\times X_{k}\rightarrow\R$. 
\end{defn}

\begin{defn}
For $S_{i}\subseteq X_{i}$, $i=1,\dots,k$ and a real number $d$,
we define the $k$-dimensional matrix 
\[
CUT(S_{1},\dots,S_{k};d)(e)=\begin{cases}
d & e=(x_{1},\dots,x_{k})\in S_{1}\times\dots\times S_{k}\\
0 & \text{otherwise}
\end{cases}
\]
\end{defn}
\begin{theorem}\cite{frieze-kannan-matrix}\label{reg-fk-higher}
Suppose $J$ is an arbitrary $k$-dimensional matrix on $X_{1}\times\dots\times X_{k}$,
where we assume that $k\geq3$ is fixed. Let $N:=|X_{1}|\times\dots\times|X_{k}|$
and let $\epsilon,\delta\in(0,1]$. Then, in time $O(k^{O(1)}\epsilon^{-O(\log_{2}k)}2^{\tilde{O}(1/\epsilon^{2})}\delta^{-2})$,
we can, with probability at least $1-\delta$, find a cut decomposition
of width $O(\epsilon^{2-2k})$, coefficient length at most $\sqrt{27}^{k}||\vec{J}||_{2}/\sqrt{N}$
and error at most $\epsilon2^{k}\sqrt{N}||\vec{J}||_{2}$. \\
\end{theorem}
As in the $k=2$ case, the error of the cut decomposition refers
to $||J-(D^{(1)}+\dots+D^{(s)})||_{\infty\mapsto1}:=\sup_{||x||_{\infty}\leq1}||J-(D^{(1)}+\dots+D^{(s)})||_{1}$,
where we view a $k$-dimensional matrix $M$ as a linear operator
from $(\R^{n})^{k-1}\rightarrow\R^{n}$ via the formula $[M(x_{1},\dots,x_{k-1})]_{i}:=\sum_{j_{1},\dots,j_{k-1}}M_{j_{1}\dots j_{k-1}i}x_{1,j_{1}}\dots x_{k-1,j_{k-1}}$. \\

Given this theorem, the proof proceeds exactly as before \textendash{}
the same argument as \cref{applying-reg-lemma} shows that we can replace $J$ by $D^{(1)}+\dots+D^{(s)}$
(possibly incurring an additive error of $\frac{2^{k}\epsilon||\vec{J}||_{1}}{\sqrt{\Delta}}$),
and once we have made this replacement, a similar optimization scheme
as before lets us deduce \cref{thm-mrf}. The proof of \cref{thm-mrf-nonconstant} is identical, the only difference being that we can obtain a cut decomposition of smaller width using the following theorem. 

\begin{theorem}\cite{alon-etal-samplingCSP}\label{reg-alon-etal}
Let $J$ be a an arbitrary $k$-dimensional matrix on $X_{1}\times\dots\times X_{k}$,
where we assume that $k\geq3$ is fixed. Let $N:=|X_{1}|\times\dots\times|X_{k}|$
and let $\epsilon>0$. Then, in time $2^{O(1/\epsilon^{2})}O(N)$
and with probability at least $0.99$, we can find a cut decomposition
width at most $4/\epsilon^{2}$, coefficient length at most $2||\vec{J}||_{2}/\epsilon\sqrt{N}$
and error at most $\epsilon2^{k}\sqrt{N}||\vec{J}||_{2}$. 
\end{theorem}

\begin{comment}
\begin{remark}
Although for simplicity, we stated \cref{thm-mrf} and \cref{thm-mrf-nonconstant} only for $k$-uniform $\Delta$-dense Markov random fields, it is immediately seen that the results extend to general $\Delta$-dense Markov random fields of order $K$ by simply applying \cref{reg-fk-higher} or \cref{reg-alon-etal} to each $J^{(k)}$, $k\in[K]$.  
\end{remark}
\end{comment}
\subsection{Ising models with low threshold rank}
\label{ltr}
As in \cite{risteski-ising}, we can also consider Ising models of low threshold
rank. For simplicity, we consider only the regular case, noting that
our results generalise with appropriate modifications to the non-regular
case as well. 

\begin{defn}
A \emph{regular weighted Ising model }is one for which
$\sum_{j}|J_{i,j}|=J'$ for all $i$. The \emph{normalized adjacency
matrix }of a regular Ising model is the matrix with
entries $J_{i,j}/J'$. We will denote this matrix by $J_{D}$.  
\end{defn}

\begin{defn}
The \emph{$\delta$-sum-of squares threshold rank} of a regular Ising model is defined to be $t_{\delta}(J_{D}):=\sum_{i:|\lambda_{i}|>\delta}\lambda_{i}^{2}$,
where $\lambda_{1},\dots,\lambda_{n}$ denote the eigenvalues of $J_{D}$. 
\end{defn}
Note that since all eigenvalues of $J_{D}$ have absolute value at
most $1$, our criterion for low threshold rank is more general than
defining $t_{\delta}(J_{D})$ to be the \emph{number }of eigenvalues
of $J_{D}$ with absolute value strictly greater than $\delta$. \\

Our definition of $t_{\delta}(J_{D})$ differs slightly from the definition
in \cite{risteski-ising}. There, the author defined $t_{\delta}(J_{D}):=\sum_{i\colon|\mu_{i}|>\delta}1,$
where $\mu_{1},\dots,\mu_{n}$ denote the eigenvalues of the matrix
$\text{abs}(J_{D})$ defined by $\text{abs}(J_{D})_{i,j}=|(J_{D})_{i,j}|$.
His definition and our definition are related via the standard linear
algebra fact that $\sum_{i=1}^{n}\lambda_{i}^{2}=||\vec{J_{D}}||_{2}^{2}$;
since $||\vec{J}_{D}||_{2}^{2}=||\vec{\text{abs}(J_{D})}\|_{2}^{2}$,
it follows that $\sum_{i=1}^{n}\lambda_{i}^{2}=\sum_{i=1}^{n}\mu_{i}^{2}$. 

The same linear algebra fact shows that the low $\delta$-sum-of-squares
rank setting strictly generalises the $\Delta$-dense setting: for
any regular, $\Delta$-dense Ising model and for any $\delta>0$,
we have $t_{\delta}(J_{D})\leq||\vec{J}_{D}||_{2}^{2}=\frac{1}{J'^{2}}||\vec{J}||_{2}^{2}\leq\frac{1}{J'^{2}}\frac{||\vec{J_{1}||_{1}^{2}}}{\Delta n^{2}}=\frac{1}{\Delta}$,
where we have used that $||\vec{J}||_{1}=nJ'$ in the regular case.\\

Our methods extend straightforwardly to the low sum-of-squares threshold
rank setting due to the following algorithmic regularity lemma of
Gharan and Trevisan.
\begin{theorem}\cite{gharan-trevisan}\label{reg-ghar-trev}
Let $J$ be the matrix of interaction strengths of a regular Ising
model, let $\epsilon>0$ and let $t:=t_{\epsilon/2}(J_{D})$. Then,
in time $poly(n,t,\frac{1}{\epsilon})$, we can find a cut decomposition
of width at most $16t/\epsilon^{2}$, $|d_{i}|\leq\frac{\sqrt{t}}{||\vec{J}||_{1}}$and
error at most $4\epsilon||\vec{J}||_{1}$. 
\end{theorem}
Recall that in the case of the $\Delta$-dense Ising model, the bounds
on the cut decomposition we actually used were that the width
is $O(\epsilon^{-2})$, each $|d_{i}|$ is at most $\frac{\sqrt{27}}{\sqrt{\Delta}n^{2}}||\vec{J}||_{1}$,
and the error is at most $\frac{4\epsilon}{\sqrt{\Delta}}||\vec{J}||_{1}$, which are quite similar since $||\vec{J_1}|| = J'n$. Hence, by exactly the same analysis as for the $\Delta$-dense case, we can conclude \cref{thm-ltr}. We omit further details.  
\iffalse
\section{Appendix: Constant-time algorithm for small $\|\vec J\|_1$}
In this section we show a simple algorithm which achieves an $\epsilon \|\vec J\|_1$ approximation
in $O(2^{2^{1/\epsilon^2\Delta}})$ time when $\|\vec J\|_1 = \omega(1)$, for $n$ sufficiently large. 
The idea is simple: as in Algorithm~\ref{convex-partition}, we will refine the partition from weak regularity 
to get a partition of the vertices into $V_1, \ldots, V_r$ where $r \le 2^{2s} = 2^{\tilde{O}(1/\epsilon^2\Delta)}$. Now, we show how to directly approximate $\log Z'$ from Lemma~\ref{applying-reg-lemma}
by guessing the proportion of up-spins in each $V_a$ up to small additive error, and approximating
the entropy contribution from each small box in this $V_a$-space. 

\begin{lemma}
Suppose that
\[ r_i = \sum_{V_a \subset R_i} z_a, \]
that
\[ r'_i = \sum_{V_a \subset R_i} z'_a \]
and $|z_a - z'_a| \le \gamma'n$ where $\gamma' = \gamma/2^{2s}$
\end{lemma}
\begin{proof}
Observe that if we let and
$|z_a - z'_a| \le \gamma'n$ then $|r_i - r'_i| \le \gamma n$.
By Lemma~\ref{lemma:gamma-def}, if $\gamma \le \frac{\epsilon \sqrt{\Delta}}{4 \sqrt{27} s}$ such that $r_i,c_i,r'_i,c'_i \le n$, 
$|r_i - r'_i| \le \gamma n$ and $|c_i - c'_i| \le \gamma n$ 
for all $i$, we get that 
$\sum_i d_i|r'_i c'_i - r_i c_i| \le \epsilon ||\vec{J}||_{1}/2$. 
\end{proof}
\fi

\section{Acknowledgements}
We thank David Gamarnik for insightful comments, Andrej Risteski for helpful discussions related to his work \cite{risteski-ising}, and Yufei Zhao for introducing us to reference \cite{alon-etal-samplingCSP}.

\bibliographystyle{plain}
\bibliography{ising-regularity,all}
\begin{comment}
\section{Appendix: Proof of \cref{thm-approx-magnetization} }\label{appendix-magnetization-proof}

\section{Appendix: Proof of \cref{thm-qualitative-tightness}}\label{appendix-tightness-proof}

\section{Appendix: Proof of \cref{z''-approx}}\label{app:z''approx}
\end{comment}
\end{document}